\documentclass{article}

\PassOptionsToPackage{numbers, compress}{natbib}

\usepackage[preprint]{neurips_2019}
\usepackage[utf8]{inputenc}
\usepackage{amsmath,amsfonts,amsthm}
\usepackage{graphicx}
\usepackage{fullpage}
\usepackage{color}
\usepackage{subcaption}

\def\bbE{\mathbb{E}}
\def\calH{\mathcal{H}}
\def\calI{\mathcal{I}}
\def\bbR{\mathbf{R}}
\def\calX{\mathcal{X}}
\def\calF{\mathcal{F}}
\def\calG{\mathcal{G}}
\def\KL{\ensuremath{\mathbb{KL}}}
\def\lin{lin}
\def\sgn{\text{sign}}
\def\calQ{\mathcal{Q}}

\def\calN{\mathcal{N}}
\def\ipm{\text{IPM}}

\newtheorem{theorem}{Theorem}
\newtheorem{lemma}{Lemma}
\newtheorem{definition}{Definition}
\newtheorem{corollary}{Corollary}

\newcommand{\dist}{\ensuremath{\mathcal{D}}}
\newcommand{\renyidiv}{\ensuremath{\mathcal{D}_{R,\alpha}}}

\usepackage[utf8]{inputenc} 
\usepackage[T1]{fontenc}    
\usepackage{hyperref}       
\usepackage{url}            
\usepackage{booktabs}       
\usepackage{amsfonts}       
\usepackage{nicefrac}       
\usepackage{microtype}      
\usepackage{cleveref}

\newcommand{\squishlist}{
	\begin{list}{$\bullet$}
		{
			\setlength{\itemsep}{0pt}
			\setlength{\parsep}{3pt}
			\setlength{\topsep}{3pt}
			\setlength{\partopsep}{0pt}
			\setlength{\leftmargin}{1.5em}
			\setlength{\labelwidth}{1em}
			\setlength{\labelsep}{0.5em} } }
	
	\newcommand{\squishend}{
\end{list}  }

\title{Capacity Bounded Differential Privacy}

\author{
Kamalika Chaudhuri\\
UC San Diego\\
\texttt{kamalika@cs.ucsd.edu}\\
\And
Jacob Imola\\
UC San Diego\\
\texttt{jimola@eng.ucsd.edu}\\
\And
Ashwin Machanavajjhala\\
Duke University\\
\texttt{ashwin@cs.duke.edu}
} 

\begin{document}

\maketitle

\begin{abstract}
Differential privacy has emerged as the gold standard for measuring the risk posed by an algorithm's output to the privacy of a single individual in a dataset. It is defined as the worst-case distance between the output distributions of an algorithm that is run on inputs that differ by a single person. 
In this work, we present a novel relaxation of differential privacy, \textit{capacity bounded differential privacy}, where the adversary that distinguishes the output distributions is assumed to be \textit{capacity-bounded} -- i.e. bounded not in computational power, but in terms of the function class from which their attack algorithm is drawn. We model adversaries of this form using restricted $f$-divergences between probability distributions, and study properties of the definition and algorithms that satisfy them. Our results demonstrate that these definitions possess a number of interesting properties enjoyed by differential privacy and some of its existing relaxations; additionally, common mechanisms such as the Laplace and Gaussian mechanisms enjoy better privacy guarantees for the same added noise under these definitions. 
\end{abstract}

\section{Introduction}

Differential privacy \cite{Dwork2014:text} has emerged as a gold standard for measuring the privacy risk posed by algorithms analyzing sensitive data. A randomized algorithm satisfies differential privacy if an arbitrarily powerful attacker is unable to distinguish between the output distributions of the algorithm when the inputs are two datasets that differ in the private value of a single person. This provides a  guarantee that the additional disclosure risk to a single person in the data posed by a differentially private algorithm is limited, even if the attacker has access to side information. However, a body of prior work~\cite{Sarwate2013:signal, Chaudhuri2011:ERM, Kifer2012:erm, Abadi2016:deep}  has shown that this strong privacy guarantee comes at a cost: for many machine-learning tasks, differentially private algorithms require a much higher number of samples to acheive the same amount of accuracy than is needed without privacy.

Prior work has considered relaxing differential privacy in a number of different
ways. Pufferfish~\cite{Kifer2014:tods} and Blowfish~\cite{He2013:blowfish}
generalize differential privacy by restricting the properties of an individual
that should not be inferred by the attacker, as well as explicitly enumerating
the side information available to the adversary.  Renyi- and KL-differential
privacy~\cite{Mironov2017:Renyi, Wang2015:generalization} measure privacy loss
as the $\alpha$-Renyi and KL-divergence between the output distributions
(respectively). The original differential privacy definition measures privacy as
a max-divergence (or $\alpha$-Renyi, with $\alpha \rightarrow \infty$).
Computational differential privacy (CDP) \cite{Mironov2009:CDP} considers a
computationally bounded attacker, and aims to ensure that the output
distributions are computationally indistinguishable. These three approaches are
orthogonal to one another as they generalize or relax different aspects of the
privacy definition.

In this paper, we consider an novel approach to relaxing differential privacy by
restricting the adversary to ``attack" or post-process the output of a private
algorithm using functions drawn from a \textit{restricted function class}. These
adversaries, that we call \textit{capacity bounded}, model scenarios where the
attacker is machine learnt and lies in some known space of functions (e.g., all
linear functions, linear classifiers, multi-layer deep networks with given
structure, etc.). A second application of this setting is a user under a
data-usage contract that restricts how the output of a private algorithm can be
used. If the contract stipulates that the user can only compute a certain class
of functions on the output, then a privacy guarantee of this form ensures that
no privacy violation can occur if users obey their contracts. Unlike
computational DP, where computationally bounded adversaries do not meaningfully
relax the privacy definition in the typical centralized differential privacy
model \cite{Groce2011:limitsCDP}, we believe that capacity bounded adversaries
will relax the definition to permit more useful algorithms and are a natural and
interesting class of adversaries.

The first challenge is how to model these adversaries. We begin by showing that
privacy with capacity bounded adversaries can be cleanly modeled through the
restricted divergences framework~\cite{Liu2018:inductive, Liu2017:approximation, Nowozin2016:fgan} that has been recently used to build a theory for generative adversarial networks. This gives us a notion of \textit{$(\mathcal{H}, \Gamma)$-capacity bounded differential privacy}, where the privacy loss is measured in terms of a divergence $\Gamma$ (e.g., Renyi) between output distributions of a mechanism on datasets that differ by a single person restricted to functions in $\mathcal{H}$ (e.g., $lin$, the space of all linear functions).

We next investigate properties of these privacy definitions, and show that they enjoy many of the good properties enjoyed by differential privacy and its relaxations -- convexity, graceful composition, as well as post-processing invariance to certain classes of functions. We analyze well-known privacy mechanisms, such as the Laplace and the Gaussian mechanism under $(lin, \KL)$ and $(lin, Renyi)$ capacity bounded privacy -- where the adversaries are the class of all linear functions. We show that restricting the capacity of the adversary does provide improvements in the privacy guarantee in many cases. We then use this to demonstrate that the popular Matrix Mechanism~\cite{Li2010:matrix, Li2012:matrix, McKenna2018:matrix} gives an improvement in the privacy guarantees when considered under capacity bounded definition. 

We conclude by showing some preliminary results that indicate that the capacity
bounded definitions satisfy a form of algorithmic generalization. Specifically, for every class of queries $\calQ$, there exists a (non-trivial) $\calH$ such that an algorithm
that answers queries in the class $\calQ$ and is $(\calH, \KL)$-capacity bounded
private with parameter $\epsilon$ also ensures generalization with parameter
$O(\sqrt{\epsilon})$. 

The main technical challenge we face is that little is known about properties of restricted divergences. While unrestricted divergences such as KL and Renyi are now well-understood as a result of more than fifty years of research in information theory, these restricted divergences are only beginning to be studied in their own right. A side-effect of our work is that we advance the information geometry of these divergences, by establishing properties such as versions of Pinsker's Inequality and the Data Processing Inequality. We believe that these will be of independent interest to the community and aid the development of the theory of GANs, where these divergences are also used. 

\section{Preliminaries}\label{sec:background}

\subsection{Privacy} 

Let $D$ be a dataset, where each data point represents a single person's value.
A randomized algorithm $A$ satisfies differential privacy~\cite{Dwork2014:text}
if its output is insensitive to adding or removing a data point to its input
$D$. We can define this privacy notion in terms of the Renyi Divergence of two
output distributions: $A(D)$ -- the distribution of outputs generated by $A$
with input $D$, and $A(D')$, the distrbution of outputs generated by $A$ with
input $D'$, where $D$ and $D'$ differ by a single person's
value~\cite{Mironov2017:Renyi}. Here, recall that the Renyi divergence of order
$\alpha$ between distributions $P$ and $Q$ can be written as: $ \renyidiv (P, Q) = \frac{1}{\alpha - 1} \log \left( \int_x P(x)^{\alpha} Q(x)^{1 - \alpha} dx\right).$

\begin{definition} [Renyi Differential Privacy]
A randomized algorithm $A$ that operates on a dataset $D$ is said to provide $(\alpha, \epsilon)$-Renyi differential privacy if for all $D$ and $D'$ that differ by a single person's value, we have: $\dist_{R, \alpha}(A(D), A(D')) \leq \epsilon.$
\label{def:renyidp}
\end{definition}

When the order of the divergence $\alpha \rightarrow \infty$, we require the
max-divergence of the two distrbutions bounded by $\epsilon$ -- which is
standard differential privacy~\cite{Dwork2006:dp}. When $\alpha \rightarrow 1$,
$\renyidiv$ becomes the Kullback-Liebler (KL) divergence, and we get KL
differential privacy~\cite{Wang2016:klprivacy}.

\subsection{Divergences and their Variational Forms} 

A popular class of divergences is Czisar's $f$-divergences~\cite{Csiszar1967},
defined as follows.
\begin{definition}
Let $f$ be a lower semi-continuous convex function such that $f(1) = 0$, and let
  $P$ and $Q$ be two distributions over a probability space $(\Omega, \Sigma)$
  such that $P$ is absolutely continuous with respect to $Q$. Then, the
  $f$-divergence between $P$ and $Q$, denoted by $\dist_f(P, Q)$ is defined as:
  $\dist_f(P, Q) = \int_{\Omega} f\left(\frac{dP}{dQ}\right) dQ$.
\end{definition}

Examples of $f$-divergences include the KL divergence ($f(t) = t \log t$), the
total variation distance ($f(t) = \frac{1}{2}|t - 1|$) and $\alpha$-divergence
($f(t) = (|t|^{\alpha} - 1)/(\alpha^2 - \alpha)$). 

Given a function $f$ with domain $\bbR$, we use $f^*$ to denote its Fenchel conjugate: $f^*(s) = \sup_{x \in \bbR} x \cdot s - f(x).$  \cite{Nguyen2010} shows that $f$-divergences have a dual variational form:
\begin{equation}\label{eqn:fdivdual}
 \dist_f(P, Q) = \sup_{h \in \mathcal{F}} \bbE_{x \sim P} [h(x)] - \bbE_{x \sim Q}[f^*(h(x))], 
\end{equation}
where $\mathcal{F}$ is the set of all functions over the domain of $P$ and $Q$.

\paragraph{Restricted Divergences.} Given an $f$-divergence and a class of functions $\calH \subseteq \mathcal{F}$, we can define a notion of a $\calH$-restricted $f$-divergence by selecting, instead of $\mathcal{F}$, the more restricted class of functions $\calH$, to maximize over in~\eqref{eqn:fdivdual}:
\begin{equation}\label{eqn:fdivrestricted}
 \dist^{\calH}_f(P, Q) = \sup_{h \in \calH} \bbE_{x \sim P} [h(x)] - \bbE_{x \sim Q}[f^*(h(x))], 
\end{equation}
These restricted divergences have previously been considered in the context of, for example, Generative Adversarial Networks~\cite{Nowozin2016:fgan, Arora2017, Liu2017:approximation, Liu2018:inductive}.

While Renyi divergences are not $f$-divergences in general, we can also define
restricted versions for them by going through the corresponding
$\alpha$-divergence -- which, recall, is an $f$-divergence with $f(t) =
(|t|^{\alpha} - 1)/(\alpha^2 - \alpha)$, and is related to the Renyi divergence
by a closed form equation~\cite{Cichocki2010:divergences}. Given a function
class $\calH$, an order $\alpha$, and two probability distributions $P$ and $Q$,
we can define the $\calH$-restricted Renyi divergence of order $\alpha$ using
the same closed form equation on the $\calH$-restricted $\alpha$-divergence as
follows:
\begin{equation} \label{eqn:alpharenyi}
\renyidiv^{\calH} (P, Q) = \left( \log \left( 1 + \alpha (\alpha - 1) \dist_{\alpha}^{\calH} (P, Q) \right)\right)/(\alpha - 1)  
\end{equation}
where $\dist_{\alpha}^{\calH}$ is the corresponding $\calH$-restricted $\alpha$-divergence.

\section{Capacity Bounded Differential Privacy}\label{sec:cbp}

The existence of $\calH$-restricted divergences suggests a natural notion of
privacy -- when the adversary lies in a (restricted) function class $\calH$, we
can, instead of $\calF$, consider the class $\calH$ of functions in the
supremum. This enforces that no adversary in the function class $\calH$ can
distinguish between $A(D)$ and $A(D')$ beyond $\epsilon$. We call these
\emph{capacity bounded adversaries}.

\begin{definition}[$(\calH, \Gamma)$-Capacity Bounded Differential Privacy]
  \label{def:cbp}
Let $\calH$ be a class of functions with domain $\calX$, and $\Gamma$ be a divergence. A mechanism $A$ is said to offer $(\calH, \Gamma)$-capacity bounded privacy with parameter $\epsilon$ if for any two $D$ and $D'$ that differ by a single person's value, the $\calH$-restricted $\Gamma$-divergence between $A(D)$ and $A(D')$ is at most $\epsilon$: 
\[ \Gamma^{\calH}(A(D), A(D')) \leq \epsilon \] 
\end{definition}

When $\calH$ is the class of all functions, and $\Gamma$ is a Renyi divergence, the definition reduces to Renyi Differential privacy; capacity bounded privacy is thus a generalization of Renyi differential privacy. 

\paragraph{Function Classes.} The definition of capacity bounded privacy allows for an infinite number of variations corresponding to the class of adversaries $\calH$.
 
An example of such a class is all linear adversaries over a feature space $\phi$, which includes all linear regressors over $\phi$. A second example is the class of all functions in an Reproducible Kernel Hilbert Space; these correspond to all kernel classifiers. A third interesting class is linear combinations of all Relu functions; this correspond to all two layer neural networks. These function classes would capture typical machine learnt adversaries, and designing mechanisms that satisfy capacity bounded DP with respect to these functions classes is an interesting research direction.




\section{Properties}

The success of differential privacy has been attributed its highly desirable
properties that make it amenable for practical use. In particular,
\cite{Kifer2012} proposes that any privacy definition should have three
properties -- convexity, post-processing invariance and graceful composition --
all of which apply to differential privacy. We now show that many of these
properties continue to hold for the capacity bounded definitions. The proofs
appear in Appendix~\ref{sec:app-properties}.

\paragraph{Post-processing.} Most notions of differential privacy satisfy
post-processing invariance, which states that applying any function to the
output of a private mechanism does not degrade the privacy guarantee. We cannot
expect post-processing invariance to hold with respect to all functions for
capacity bounded privacy -- otherwise, the definition would be equivalent to
privacy for all adversaries!

However, we can show that for any $\calH$ and for any $\Gamma$, $(\calH,
\Gamma)$-capacity bounded differential privacy is preserved after
post-processing if certain conditions about the function classes hold:

\begin{theorem}\label{thm:postprocessing}
Let $\Gamma$ be an $f$-divergence or the Renyi divergence of order $\alpha > 1$,
  and let $\calH$ $\calG$, and $\calI$ be function classes such that for any
  $g \in \calG$ and $i \in \calI$, $i \circ g \in \calH$. If algorithm $A$ satisfies
  $(\calH, \Gamma)$-capacity bounded privacy with parameter $\epsilon$, then,
  for any $g \in \calG$, $g \circ A$ satisfies $(\calI, \Gamma)$-capacity
  bounded privacy with parameter $\epsilon$.
\end{theorem}

Specifically, if $\calI = \calH$, then $A$ is post-processing invariant.
Theorem~\ref{thm:postprocessing} is essentially a form of the popular Data
Processing Inequality applied to restricted divergences; its proof is in the
Appendix and follows from the definition as well as algebra. An example of
function classes $\calG$,$\calH$, and $\calI$ that satisfy this conditions is when
$\calG, \calH, \calI$ are linear functions, where $\calG : \bbR^s \rightarrow
\bbR^d$, $\calH : \bbR^s \rightarrow \bbR$, and $\calI : \bbR^d \rightarrow
\bbR$.

\paragraph{Convexity.} A second property is convexity~\cite{Kifer2010:axioms},
which states that if $A$ and $B$ are private mechanisms with privacy
parameter $\epsilon$
then so is a composite mechanism $M$ that tosses a
(data-independent) coin and chooses to run $A$ with probability $p$ and $B$
with probability $1 - p$.  We show that convexity holds for $(\calH,
\Gamma)$-capacity bounded privacy for any $\calH$ and any $f$-divergence
$\Gamma$.

\begin{theorem}\label{thm:convexity}
  Let $\Gamma$ be an $f$-divergence and $A$ and $B$ be two mechanisms which
  have the same range and
  provide $(\calH, \Gamma)$-capacity bounded privacy with parameter $\epsilon$.
  Let $M$ be a mechanism which tosses an independent coin, and then executes
  mechanism $A$ with probability $\lambda$ and $B$ with probability $1-\lambda$.
  Then, $M$ satisfies $(\calH, \Gamma)$-capacity bounded privacy with parameter
  $\epsilon$.
\end{theorem}

We remark that while differential privacy and KL differential privacy satisfy
convexity, (standard) Renyi differential privacy does not; it is not surprising
that neither does its capacity bounded version. The proof uses convexity of the function $f$ in an $f$-divergence.

\paragraph{Composition.} Broadly speaking, composition refers to how privacy properties of algorithms applied multiple times relate to privacy properties of the individual algorithms. Two styles of composition are usually considered -- sequential and parallel.

A privacy definition is said to satisfy {\em{parallel composition}} if the privacy loss obtained by applying multiple algorithms on disjoint datasets is the maximum of the privacy losses of the individual algorithms. In particular, Renyi differential privacy of any order satisfies parallel composition. We show below that so does capacity bounded privacy. 

\begin{theorem}\label{thm:parallelcomposition}
Let $\calH_1,\calH_2$ be two function classes that are convex and translation
  invariant. Let $\calH$ be the function
  class:
  \[ \calH = \{ h_1 + h_2 | h_1 \in \calH_1, h_2 \in \calH_2 \} \]
and let $\Gamma$ be the KL divergence or the Renyi divergence of order $\alpha >
1$. If mechanisms $A$ and $B$ satisfy $(\calH_1, \Gamma)$ and $(\calH_2,
  \Gamma)$ capacity bounded privacy with parameters $\epsilon_1$ and
  $\epsilon_2$ respectively, and if the datasets $D_1$ and $D_2$ are disjoint,
  then the combined release $(A(D_1), B(D_2))$ satisfies $(\calH, \Gamma)$ 
  capacity bounded privacy with parameter $\max(\epsilon_1, \epsilon_2)$. 
\end{theorem}

In contrast, a privacy definition is said to compose {\em{sequentially}} if the
privacy properties of algorithms that satisfy it degrade gracefully as the same
dataset is used in multiple private releases. In particular, Renyi differential
privacy is said to satisfy sequential additive composition -- if multiple
private algorithms are used on the same dataset, then their privacy parameters
add up. We show below that a similar result can be shown for $(\calH,
\Gamma)$-capacity bounded privacy when $\Gamma$ is the KL or the Renyi
divergence, and $\calH$ satisfies some mild conditions.

\begin{theorem}\label{thm:sequentialcomposition}
Let $\calH_1$ and $\calH_2$ be two function classes that are convex, translation
  invariant, and that includes a constant function. Let $\calH$ be the function class:
\[ \calH = \{ h_1 + h_2 | h_1 \in \calH_1, h_2 \in \calH_2 \} \]
and let $\Gamma$ be the KL divergence or the Renyi divergence of order $\alpha >
  1$. If mechanisms $A$ and $B$ satisfy $(\calH_1, \Gamma)$ and $(\calH_2,
  \Gamma)$ capacity bounded privacy with parameters $\epsilon_1$ and
  $\epsilon_2$ respectively, then the combined release $(A, B)$ satisfies
  $(\calH, \Gamma)$ capacity bounded privacy with parameter $\epsilon_1 +
  \epsilon_2$.
\end{theorem}

The proof relies heavily on the relationship between the restricted and
unrestricted divergences, as shown in~\cite{Liu2018:inductive,
Liu2017:approximation, Farnia2018:gan}, and is provided in the Appendix. Observe
that the conditions on $\calH_1$ and $\calH_2$ are rather mild, and include a
large number of interesting functions. One such example of $\calH$ is the set of
ReLU neural networks with linear output node, a common choice when performing
neural network regression.

The composition guarantees offered by Theorem~\ref{thm:sequentialcomposition}
are non-adaptive -- the mechanisms $A$ and $B$ are known in advance, and
$B$ is not chosen as a function of the output of $A$. Whether fully general
adaptive composition is possible for the capacity bounded definitions is left as
an open question for future work.

\section{Privacy Mechanisms}\label{sec:mechanisms}

The definition of capacity bounded privacy allows for an infinite number of
variations, corresponding to the class of adversaries $\calH$ and divergences
$\Gamma$, exploring all of which is outside the scope of a single paper. For the
sake of concreteness, we consider {\em{linear}} and (low-degree)
{\em{polynomial}} adversaries for $\calH$ and KL as well as Renyi divergences 
of order $\alpha$ for $\gamma$. These correspond to cases where a linear or a 
low-degree polynomial function is used by an adversary to attack privacy.


A first sanity check is to see what kind of linear or polynomial guarantee is offered
by a mechanism that directly releases a non-private value (without any added
randomness). This mechanism offers no finite linear KL or Renyi differential
privacy parameter -- which is to be expected from any sensible privacy
definition (see Lemma~\ref{lemma:non-private} in the Appendix).

\begin{table}

  \begin{tabular}[pos]{|l|l|p{55mm}|p{35mm}|}
  \hline
  Divergence & Mechanism & Privacy Parameter, Linear Adversary  & Privacy
  Parameter, Unrestricted \\ \hline
  KL & Laplace 
     & $ \sqrt{1+\epsilon^2}-1 +
       \log\left(1-\frac{(\sqrt{1+\epsilon^2}-1)^2}{\epsilon^2}\right) $
     & $ \epsilon-1 + e^{-\epsilon} $ \\ \hline
  KL & Gaussian & $ \nicefrac{1}{2\sigma^2} $ & $ \nicefrac{1}{2\sigma^2} $ \\ \hline
  $\alpha$-Renyi & Laplace 
    & $ \leq \frac{1}{\alpha-1}\log(1+2^{\alpha-1}\epsilon^\alpha)$
    & $ \geq \epsilon - \nicefrac{\log(2)}{\alpha-1}$ \\ \hline
  $\alpha$-Renyi & Gaussian 
    & $ \leq \frac{1}{\alpha-1}
      \log(1+\nicefrac{\sqrt{2\pi}^{\alpha-1}}{\sigma^\alpha})$
    & $ \nicefrac{\alpha}{2\sigma^2} $ \\ \hline
  $\alpha$-Renyi & Laplace, $d$-dim 
    & $\leq \frac{1}{\alpha-1} \log ( 1 + 2^{d(\alpha-1)} (\epsilon 
    \|v\|_\alpha)^\alpha )$ 
    & $\geq \epsilon\|v\|_1 - \nicefrac{d \log(2)}{\alpha-1}$ \\ \hline
  $\alpha$-Renyi & Gaussian, $d$-dim 
    & $\leq \frac{1}{\alpha-1} \log \left( 1 + \frac{2^{d(\alpha-1)}
    \sqrt{\pi/2}^{\alpha-1} \|v\|_\alpha^\alpha}{\sigma^\alpha} \right)$ 
    & $ \frac{\alpha \|v\|_2^2}{2\sigma^2} $ \\ \hline
\end{tabular}
  \vspace{1em}
\caption{Privacy parameters of different mechanisms and divergences with a
  linear adversary and unrestricted. Proofs appear in
  Appendix~\ref{sec:mechanisms-app}} \label{tab:divergences}
\end{table}

We now look at the capacity bounded privacy properties of the familiar Laplace
and Gaussian mechanisms which form the building blocks for much of differential
privacy. Bounds we wish to compare appear in Table~\ref{tab:divergences}.

\paragraph{Laplace Mechanism.} The Laplace mechanism
adds $Lap(0, 1/\epsilon)$ noise to a function with
global sensitivity $1$. In $d$ dimensions, the mechanism adds
$d$ i.i.d. samples from $Lap(0, 1/\epsilon)$ to a function with $L_1$
sensitivity 1. More generally, we consider functions whose global sensitivity
along coordinate $i$ is $v_i$. We let $v = (v_1,v_2,\ldots, v_d)$.

Table~\ref{tab:divergences} shows $(lin,KL)$-capacity bounded privacy and 
KL-DP parameters for the Laplace mechanism. The former has a slightly smaller
parameter than the latter.

Table~\ref{tab:divergences} also contains an upper bound on the $(lin, Renyi)$
capacity bounded privacy, and a lower bound on the Renyi-DP. The exact value
of the Renyi-DP is:
  \begin{equation}\label{eq:renyi-laplace}
    \frac{1}{\alpha-1}\log\left(\left(\frac{1}{2} + \frac{1}{4\alpha-2}\right)
    e^{(\alpha-1)\epsilon} +
    \left(\frac{1}{2} - \frac{1}{4\alpha-2}\right)e^{-\alpha\epsilon}\right)
  \end{equation}
By multiplying by $\alpha-1$ and exponentiating, we see that the $(lin,Renyi)$
upper bound grows with $1+\epsilon(2\epsilon)^{\alpha-1}$, while the Renyi-DP
lower bound grows with $(e^\epsilon)^{\alpha-1}$. This means
no matter what $\epsilon$ is, a moderately-sized $\alpha$ will make the
$(lin,Renyi)$ upper bound smaller than the Renyi lower bound.

Figure~\ref{fig:lin-renyi-laplace} plots the $(lin,Renyi)$ upper bound,
(\ref{eq:renyi-laplace}), and the exact value of the $(lin,Renyi)$ parameter,
 as functions of $\alpha$ when $\epsilon=1$. 
We see the exact $(lin,Renyi)$ is always better
than~(\ref{eq:renyi-laplace}), although the upper bound may sometimes be worse.
The upper bound overtakes the lower bound when $\alpha \approx 3.3$.

For the multidimensional Laplace Mechanism, the story is the same.
The $(lin, Renyi)$ upper bound can now be thought of as a function of $\epsilon
\|v\|_\alpha$, and the $Renyi$ lower bound a function of $\epsilon \|v\|_1$.
Because $\|v\|_\alpha \leq \|v\|_1$, we can replace $\|v\|_\alpha$ with
$\|v\|_1$ in the $(lin, Renyi)$ upper bound, and repeat the analysis for the
unidimensional case.
Notice that our $(lin, Renyi)$ upper bound is slightly better than using
composition $d$ times on the unidimensional Laplace mechanism which would 
result in a multiplicative factor of $d$.

Figure~\ref{fig:poly-renyi-laplace} contains plots of the exact
$(poly, Renyi)$ paramters for degree 1,2, and 3 polynomials, as functions of
$\alpha$ when $\epsilon=1$. As we expect, as
the polynomial complexity increases, the $(poly, Renyi)$ parameters
converge to the Renyi-DP parameter. This also provides an
explanation for the counterintuitive observation
that the $(poly, Renyi)$ parameters
eventually decrease with $\alpha$. The polynomial function classes are too
simple to distinguish the two distributions for larger $\alpha$, 
but their ability to do so
increases as the polynomial complexity increases.

\begin{figure}
  \begin{subfigure}[t]{0.48\textwidth}
    \centering
    \includegraphics[scale=0.4]{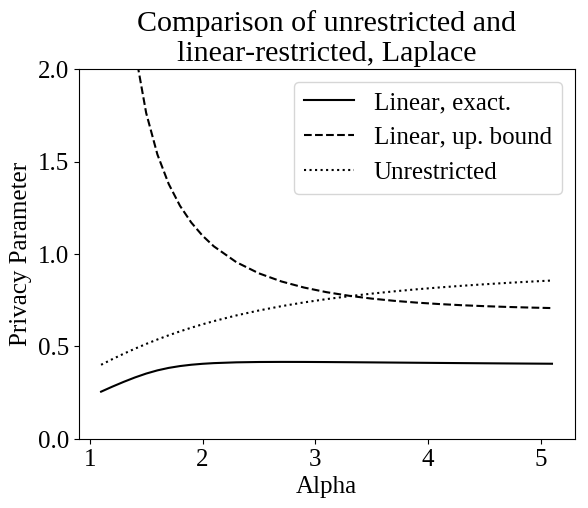}
    \caption{Plots of $(lin,Renyi)$ capacity bounded DP 
      and Renyi-DP parameters for Laplace mechanism when $\epsilon=1$. For
      $(lin,Renyi)$, the upper bound and exact value
      are shown. }\label{fig:lin-renyi-laplace}
  \end{subfigure}
  \hfill
  \begin{subfigure}[t]{0.48\textwidth}
    \centering
    \includegraphics[scale=0.4]{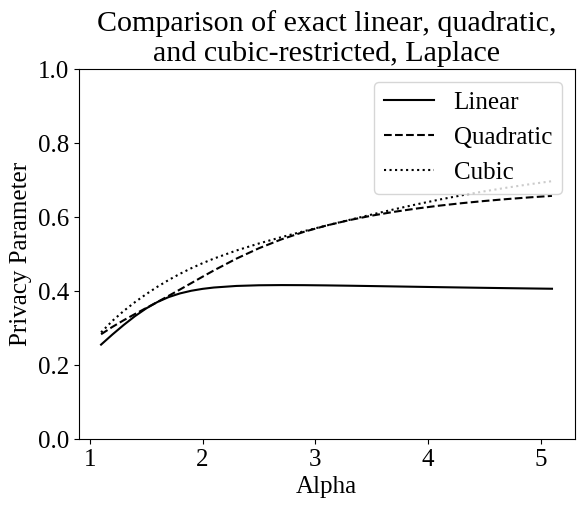}
    \caption{Comparison of exact values of $(poly, Renyi)$ capacity bounded DP
    parameters for Laplace mechanism when $\epsilon=1$.
    }\label{fig:poly-renyi-laplace}
  \end{subfigure}
\end{figure}

\paragraph{Gaussian Mechanism.} The Gaussian mechanism
adds $\calN(0, \sigma^2)$ noise to a function with global sensitivity $1$. In
$d$ dimensions, the mechanism adds $\calN(0, \sigma^2I_d)$ to a function with
$L_2$ sensitivity 1. More generally, we consider functions whose global
sensitivity along coordinate $i$ is $v_i$ We let $v = (v_1,v_2,\ldots,v_d)$.

Whereas the $(lin, \KL)$ parameter for Laplace
is a little better than the KL-DP parameter, 
Table~\ref{tab:divergences} shows the Gaussian mechanism has the same
parameter. This is because if $P$ and $Q$ are two Gaussians with equal variance, the
function $h$ that maximizes the variational formulation corresponding to the
KL-divergence is a linear function.

For Renyi capacity bounded privacy, the observations we make are nearly
identical to that of the Laplace Mechanism. The reader is referred to 
Appendix~\ref{sec:app-gauss} for plots and specific details.

\subsubsection*{Matrix Mechanism.}
Now, we show how to use the bounds in Table~\ref{tab:divergences} to obtain
better capacity bounded parameters for a composite mechanism often used in
practice: the Matrix mechanism~\cite{Li2010:matrix,Li2012:matrix,
McKenna2018:matrix}.
The Matrix mechanism is a very general method of computing linear queries on a
dataset, usually with less error than the Laplace Mechanism. Given a dataset $D \in
\Sigma^m$ over a finite alphabet $\Sigma$ of size $n$, we can form a vector of
counts $x \in\bbR^{n}$ such that $x_i$ contains how many times $i$ appears in $D$.
A linear query is a vector $w \in \bbR^{n}$ and has answer $w^Tx$.
A set of $d$ linear queries can then be given by a matrix $W \in \bbR^{d \times
n}$ with the goal of computing $Wx$ privately.

A naive way to do this is to use the Laplace Mechanism in $d$ dimensions to
release $x$ and then multiply by $W$. The key
insight is that, for any $A \in \bbR^{s \times n}$ of rank $n$,
we can instead add noise 
to $Ax$ and multiply the result by $WA^\dag$. Here, $A^\dag$ denotes the
pseudoinverse of $A$ such that $WA^\dag A = W$.

The
Laplace Mechanism arises as the special case $A = I$; however, more carefully 
chosen $A$s may be used to get privacy with less noise. This gives 
rise to the Matrix mechanism:
\[
  M_A(W,x,\epsilon) = 
  WA^\dag(Ax + \|A\|_1Lap_d(0,1/\epsilon))
\]
Here, $\|A\|_1$ is the maximum $L_1$-norm of any column of $a$.
Prior work shows that this mechanism provides differential privacy and suggest
different methods for picking an $A$. Regardless of which $A$ is chosen,
, we are able to provide a capacity-bounded privacy parameter that is better
than any known Renyi-DP analysis has shown:

\begin{theorem}[Matrix Mechanism]\label{thm:matrix-mechanism}
  Let $x \in \bbR^n$ be a data vector, $W \in \bbR^{d \times n}$ be a query matrix,
  and $A \in \bbR^{s \times n}$ be a strategy matrix. Then, 
  releasing $M_A(W,x,\epsilon)$ offers $(lin,Renyi)$ capacity bounded privacy 
  with parameter at most
  $\frac{1}{\alpha-1}\log(1+2^{s(\alpha-1)}\epsilon^\alpha)$.
\end{theorem}
Note this is the same upper bound as the $d$-dimensional Laplace mechanism;
indeed, the proof works by applying post-processing to the Laplace mechanism.

\section{Algorithmic Generalization}
\label{sec:generalization}

Overfitting to input data has long been the curse of many statistical and machine learning methods; harmful effects of overfitting can range from poor performance at deployment time all the way up to lack of reproducibility in scientific research due to $p$-hacking~\cite{head2015extent}. Motivated by these concerns, a recent line of work in machine learning investigates properties that algorithms and methods should possess so that they can automatically guarantee generalization~\cite{Russo2015:generalization, Feldman2017:generalization, Dwork2014, Smith2015}. In this connection, differential privacy and many of its relaxations have been shown to be highly successful; it is known for example, that if adaptive data analysis is done by a differentially private algorithm, then the results automatically possess certain generalization guarantees.

A natural question is whether these properties translate to the capacity bounded differential privacy, and if so, under what conditions. We next investigate this question, and show that capacity bounded privacy does offer promise in this regard. A more detailed investigation is left for future work. 

\paragraph{Problem Setting.} More specifically, the problem setting is as follows.~\cite{Russo2015:generalization, Feldman2017:generalization, Dwork2014, Smith2015}. We are given as input a data set $S = \{ x_1, \ldots, x_n \}$ drawn from an (unknown) underlying distribution $D$ over an instance space $\calX$, and a set of ``statistical queries'' $\calQ$; each statistical query $q \in \calQ$ is a function $q: \calX \rightarrow [0, 1]$. 

A data analyst $M$ observes $S$, and then picks a query $q_S \in \calQ$ based on her observation; we say that $M$ {\em{generalizes well}} if the query $q_S$ evaluated on $S$ is close to $q_S$ evaluated on a fresh sample from $D$ (on expectation); more formally, this happens when the {\em{generalization gap}} $ \frac{1}{n} \sum_{i=1}^{n} q_S(x_i) - \bbE_{x \sim D}[q_S(x)]$
is low. 

Observe that if the query was picked without an apriori look at the data $S$, then the problem would be trivial and solved by a simple Chernoff bound. Thus bounding the generalization gap is challenging because the choice of $q_S$ depends on $S$, and the difficulty lies in analyzing the behaviour of particular methods that make this choice. 


\paragraph{Our Result.} Prior work in generalization theory~\cite{Russo2015:generalization,
Feldman2017:generalization, Dwork2014, Smith2015} shows that if $M$ possesses certain algorithmic stability properties -- such as differential privacy as well as many of its relaxations and generalizations -- then the gap is low. We next show that provided the adversarial function class $\calH$ satisfies certain properties with respect to the statistical query class $\calQ$, $(\calH, \lin)$-capacity bounded privacy also has good generalization properties. 

\begin{theorem}[Algorithmic Generalization]
\label{thm:generalize}
Let $S$ be a sample of size $n$ drawn from an underlying data distribution $D$ over an instance space $\calX$, and let $M$ be a (randomized) mechanism that takes as input $S$, and outputs a query $q_S$ in a class $\calQ$. For any $x \in \calX$, define a function $h_x: \calQ \rightarrow [0, 1]$ as: $h_x(q) = q(x)$, and let $\calH$ be any class of functions that includes $\{ h_x  | x \in \calX \}$. 

If the mechanism $M$ satisfies $(\calH, \KL)$-capacity bounded privacy with parameter $\epsilon$, then, for every distribution $D$, we have: $\Big{|} \bbE_{S \sim D, M}\left(\frac{1}{n} \sum_{i=1}^{n} q_S(x_i) - \bbE_{x \sim D}[q_S(x)]\right) \Big{|} \leq 8 \sqrt{\epsilon}.$
\end{theorem}

We remark that the result would not hold for arbitrary $(\calH, \KL)$-capacity bounded privacy, and a condition that connects $\calH$ to $\calQ$ appears to be necessary. However, for specific distributions $D$, fewer conditions may be needed. 

Observe also that Theorem~\ref{thm:generalize} only provides a bound on expectation. Stronger guarantees -- such as high probability bounds as well as adaptive generalization bounds -- are also known in the adaptive data analysis literature. While we believe similar bounds should be possible in our setting, proving them requires a variery of information-theoretic properties of the corresponding divergences, which are currently not available for restricted divergences. We leave a deeper investigation for future work.




\paragraph{Proof Ingredient: A Novel Pinsker-like Inequality.} We remark that an ingredient in the proof of Theorem~\ref{thm:generalize} is a novel Pinsker-like inequality for restricted KL divergences, which was previously unknown, and is presented below (Theorem~\ref{thm:pinsker}). We believe that this theorem may be of independent interest, and may find applications in the theory of generative adversarial networks, where restricted divergences are also used. 

We begin by defining an integral probability metric (IPM)~\cite{Gretton2009} with respect to a function class $\calH$. 

\begin{definition}
Given a function class $\calH$, and any two distributions $P$ and $Q$, the Integral Probability Metric (IPM) with respect to $\calH$ is defined as follows: $\ipm^{\calH}(P, Q) = \sup_{h \in \calH} | \bbE_P[h(x)] - \bbE_Q[h(x)] |.$
\end{definition}

Examples of IPMs include the total variation distance where $\calH$ is the class of all functions with range $[0,1]$, and the Wasserstein distance where $\calH$ is the class of all $1$-Lipschitz functions. With this definition in hand, we can now state our result. 

\begin{theorem}[Pinsker-like Inequality for Restricted KL Divergences]
Let $\calH$ be a convex class of functions with range $[-1,1]$ that is translation invariant and closed under negation. Then, for any $P$ and $Q$ such that $P$ is absolutely continuous with respect to $Q$, we have that: $\ipm^{\calH}(P, Q) \leq 8 \cdot \sqrt{\KL^{\calH}(P, Q)}.$
\label{thm:pinsker}
\end{theorem}

This theorem is an extended version of the Pinsker Inequality, which states that the total variation distance $TV(P, Q) \leq \sqrt{2 \KL(P, Q)}$; however, instead of connecting the total variation distance and KL divergences, it connects IPMs and the corresponding restricted KL divergences. 

\section{Conclusion}
\label{sec:conclusion}

We initiate a study into capacity bounded differential privacy -- a relaxation
of differential privacy against adversaries in restricted function classes. We
show how to model these adversaries cleanly through the recent framework of
restricted divergences. We then show that the definition satisfies
privacy axioms, and permits mechanisms that have higher utility (for the same
level of privacy) than regular KL or Renyi differential privacy when the
adverary is limited to linear functions. Finally, we show some preliminary
results that indicate that these definitions offer good generalization
guarantees. 

There are many future directions. A deeper investigation into novel mechanisms
that satisfy the definitions, particularly for other function classes such as
threshold and relu functions remain open. A second question is a more detailed
investigation into statistical generalization -- such as generalization in high
probability and adaptivity -- induced by these notions. Finally, our work
motivates a deeper exploration into the information geometry of adversarial
divergences, which is of wider interest to the community.

\subsubsection*{Acknowledgments.} We thank Shuang Liu and Arnab Kar for early discussions. KC and JI thank ONR under N00014-16-1-261, UC Lab Fees under LFR 18-548554 and NSF under 1253942 and 1804829 for support. AM was supported by the National Science Foundation under grants 1253327, 1408982; and by DARPA and SPAWAR under contract N66001-15-C-4067.
{\small
\bibliographystyle{plain}
\bibliography{capacity}
}

\clearpage
\appendix

\section{Analysis of Gaussian Mechanism}\label{sec:app-gauss}
\begin{figure}[h]
  \begin{subfigure}[t]{0.48\textwidth}
    \centering
    \includegraphics[scale=0.4]{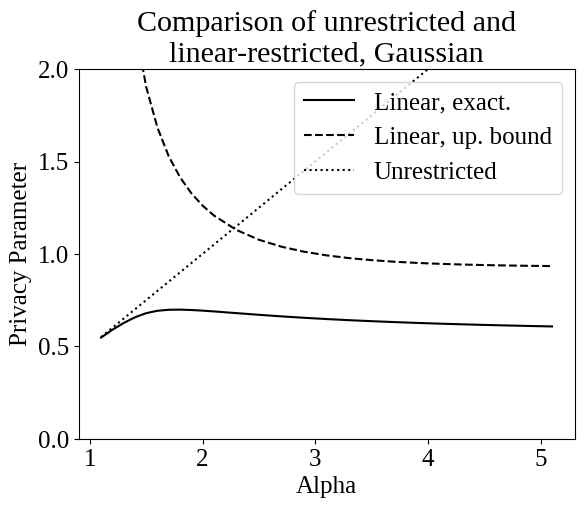}
    \caption{Plots of $(lin,Renyi)$ capacity bounded DP and Renyi-DP parameters
    for Gaussian mechanism when $\sigma=1$. For $(lin,Renyi)$, the upper bound
    and exact value are shown.}\label{fig:lin-renyi-normal}
  \end{subfigure}
  \hfill
  \begin{subfigure}[t]{0.48\textwidth}
    \centering
    \includegraphics[scale=0.4]{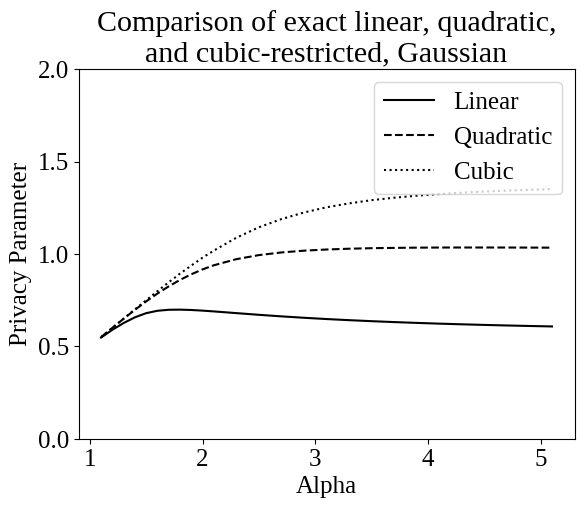}
    \caption{Comparison of exact values of $(poly, Renyi)$ capacity bounded DP
    parameters for Gaussian mechanism when $\sigma=1$.}\label{fig:poly-renyi-normal}
  \end{subfigure}
\end{figure}

Table~\ref{tab:divergences} contains an upper bound on the $(lin, Renyi)$
capacity bounded privacy parameter for the Gaussian mechanism and the 
Renyi-DP parameter.
By multiplying by $\alpha-1$ and exponentiating, we see that the $(lin, Renyi)$
upper bound grows with
$\frac{1}{\sigma}\left(\frac{\sqrt{2\pi}}{\sigma}\right)^{\alpha-1}$ while the Renyi
parameter grows with $(e^{\alpha/(2\sigma^2)})^{\alpha-1}$. Because
$\frac{\sqrt{2\pi}}{\sigma} < e^{\alpha/(2\sigma^2)}$ for all $\sigma$, we
conclude that modestly-sized values of $\alpha$ will cause the $(lin, Renyi)$ 
upper bound to fall below the Renyi parameter for all $\sigma$.

Figure~\ref{fig:lin-renyi-normal} plots the $(lin,Renyi)$ upper bound, the exact
$(lin,Renyi)$ parameter, and the Renyi-DP parameter as functions of $\alpha$
when $\sigma=1$. We see the exact $(lin,Renyi)$ is always better than
the Renyi-DP parameter, although the upper bound is worse for small $\alpha$.
The upper bound overtakes the Renyi-DP parameter when $\alpha \approx 2.3$.

For the multidimensional Gaussian mechanism, the story is mostly the same. 
Note that the $(lin, Renyi)$ upper bound can be written as
\[
    \frac{1}{\alpha-1} \log \left( 1 + 2^{d(\alpha-1)}
    \sqrt{\pi/2}^{\alpha-1} \left(
      \frac{\|v\|_\alpha}{\sigma}\right)^\alpha \right)
\]
The Renyi parameter, on the other hand, is
\[
  \alpha \left(\frac{\|v\|_2}{\sigma}\right)^2
\]
Notice that these are the same functions we looked at for the unidimensional
case, but instead of $\frac{1}{\sigma}$, they are in terms of
$\frac{\|v\|_\alpha}{\sigma}$ and $\frac{\|v\|_2}{\sigma}$, respectively.
Indeed this is no accident, because the unidimensional cases assumed the $L_2$
sensitivity of the function was 1. However, when $\alpha > 2$, we have
$\|v\|_\alpha < \|v\|_2$, so we can replace $\|v\|_\alpha$ with $\|v\|_2$, and
the $(lin, Renyi)$ upper bound will only increase. But this still gives us the
same conclusion as the unidimensional Gaussian mechanism, since now both parameters are
functions of $\frac{\|v\|_2}{\sigma}$.

Finally, our multidimensional $(lin, Renyi)$ upper bound is slightly better than
composing the Gaussian mechanism $d$ times which would result in a
multiplicative factor of $d$.

Figure~\ref{fig:poly-renyi-normal} contains plots of the exact $(poly, Renyi)$
parameters for degree 1,2, and 3 polynomials, as functions of $\alpha$ when
$\sigma=1$. As we expect, as the polynomial complexity increases, the $(poly,
Renyi)$ parameters converge to the Renyi-DP parameter. This also provides an
explanation for the counterintuitive observation that the $(poly, Renyi)$
parameters eventually decrease with $\alpha$. The polynomial function classes
are too simple to distinguish the two distributions for larget $\alpha$, but
their ability to do so increases as the polynomial complexity increases.

\section{Post-Processing, Convexity, and Composition}\label{sec:app-properties}

\begin{proof} (Of Theorem~\ref{thm:postprocessing})
It suffices to show Post-Processing Invariance for a restricted $f$-divergence.
Let $A$ be a mechanism that maps a dataset $D$ into an output $x$ in an instance space $X$. Let $P = \Pr(A(D) = \cdot)$ and $Q = \Pr(A(D') = \cdot)$. 

Suppose $g$ is a function in $\calG$ which maps an $x \in X$ into a $y \in Y$ --
  that is $y = g(x)$. Let $P'$ and $Q'$ be the distributions induced on $Y$ by
  $P$ and $Q$ respectively when we map $x$ into $y$. To show post-processing,
   we need to show that $\dist_f^{\calI}(P', Q') \leq \dist_f^{\calH}(P, Q)$.

To see this, observe that:
\begin{eqnarray*}
  \dist_f^\calI(P', Q') & = & \sup_{i \in \calI} \bbE_{P'}[i(y)] - \bbE_{Q'}[f^*(i(y)] \\
  & = & \sup_{i \in \calI} \bbE_P[i \cdot g(x)] - \bbE_Q[f^* (i \cdot g(x))] \\
  & \leq & \sup_{h \in \calH} \bbE_P[h(x)] - \bbE_Q[f^*(h(x))]
\end{eqnarray*}
where the first step follows because $y = g(x)$ and the second step follows
  because $i \cdot g \in \calH$. The theorem follows from observing that the
  right hand side in the third line is exactly $\dist_f^\calH(P, Q)$.
\end{proof}

\begin{proof} (Of Theorem~\ref{thm:convexity})
  To prove convexity, it suffices to show that 
  \[ \dist_{\alpha}^\calH(M(D), M(D')) \leq \lambda \dist_{\alpha}^{\calH}(A(D), A(D'))
  + (1 - \lambda) \dist_{\alpha}^{\calH}(B(D), B(D')) \]
	Observe that $\dist_f^{\calH}(M(D), M(D'))$ is equal to:
	\begin{eqnarray*}
		& = & \sup_{h \in \calH} \bbE_{x \sim M(D)} [h(x)] - \bbE_{x \sim
    M(D')}[f^*(h(x))] \\
		& = & \sup_{h \in \calH}  \bbE_{x \sim \lambda A(D) + (1 - \lambda) B(D) }
    [h(x)] - \bbE_{x \sim \lambda A(D') + (1 - \lambda) B(D')} [f^*(h(x))] \\
		& = & \sup_{h \in \calH} \lambda \bbE_{x \sim A(D)} [h(x)] + (1 - \lambda)
    \bbE_{x \sim B(D)} [ h(x)] - 
      \bbE_{x \sim \lambda A(D') + (1 - \lambda) B(D')} [f^*(h(x))] \\
		& \leq & \sup_{h \in \calH} \lambda \bbE_{x \sim A(D)} [h(x)] + (1 -
    \lambda) \bbE_{x \sim B(D)} [ h(x)] - \lambda \bbE_{x \sim A(D')}
    [f^*(h(x))] - (1 - \lambda) \bbE_{x \sim B(D')}[f^*(h(x))] \\
		& = & \sup_{h \in \calH} \lambda (\bbE_{x \sim A(D)} [h(x)] - \bbE_{x
    \sim A(D')} [f^*(h(x))] ) + (1 - \lambda)  (\bbE_{x \sim B(D)} [h(x)] -
    \bbE_{x \sim B(D')} [f^*(h(x))] )\\
		& \leq & \lambda \cdot \sup_{h \in \calH} \bbE_{x \sim A(D)} [h(x)] -
    \bbE_{x \sim A(D')} [f^*(h(x))] + (1 - \lambda) \cdot  \sup_{h \in \calH}
    \bbE_{x \sim B(D)} [h(x)] - \bbE_{x \sim B(D')} [f^*(h(x))] 
	\end{eqnarray*}
	
	where the second step follows from the fact that $M(D)$ is a mixture of $A(D)$
  and $B(D)$ with mixing weights $[\lambda, 1 - \lambda]$, the third step
  is a property of mixture
  distributions, the fourth step from the concavity of log, the fifth step from
  re-arrangement, and the last step from the observation that $\max_y f(y) +
  g(y) \leq \max_y f(y) + \max_y g(y)$. Observe that the last line is $\lambda
  \dist_f^{\calH}(A(D), A(D')) + (1 - \lambda) \dist_f^{\calH}(B(D), B(D'))$.
  \end{proof}

\subsection{Composition ($\calH$-bounded Renyi, KL Privacy only)}

\begin{proof}
  (Of Theorem~\ref{thm:sequentialcomposition}).
  Let $D$ and $D'$ be two datasets that differ in the private value of a single
  person, and let $(P_1,P_2) = (A(D), B(D))$ and $(Q_1, Q_2) = (A(D'), B(D'))$.
  Let $P$ be the product distribution $P_1 \otimes P_2$ and $Q$
  be the product distribution $Q_1 \otimes Q_2$.
  Finally, let $a = \alpha^2-\alpha$.
  By assumption, $\renyidiv^{\calH_i}(P_i, Q_i) \leq \epsilon_i$. Hence,
  we know $\dist_\alpha^{\calH_i}(P_i, Q_i) \leq \eta_i$ where 
  $\eta_i = \frac{\exp(\epsilon_i(\alpha-1))-1}{a}$.
  Then,
\begin{eqnarray*}
  \dist_{\alpha}^{\calH}(P, Q) & = & \inf_{P'} \dist_\alpha(P', Q) + \sup_{h \in
  \calH}\bbE_{P}[h] - \bbE_{P'}[h]\\
  & \leq & \inf_{P'=P_1' \otimes P_2'} \dist_\alpha(P', Q) + \sup_{h \in
  \calH}\bbE_{P}[h] - \bbE_{P'}[h]\\
  & = & \inf_{P'=P_1' \otimes P_2'} a\dist_\alpha(P_1', Q_1)\dist_\alpha(P_2', Q_2) +
  \dist_\alpha(P_1', Q_1) + \dist_\alpha(P_2', Q_2) + \sup_{h \in
  \calH}\bbE_{P}[h] - \bbE_{P'}[h]\\
  & \leq & \inf_{P'=P_1' \otimes P_2'} a\dist_\alpha(P_1', Q_1)\dist_\alpha(P_2', Q_2) +
  \dist_\alpha(P_1', Q_1) + \dist_\alpha(P_2', Q_2) \\
  &   & + \sup_{h_1 \in \calH_1}\bbE_{P_1}[h_1] - \bbE_{P_1'}[h_1]
      + \sup_{h_2 \in \calH_2}\bbE_{P_2}[h_2] - \bbE_{P_2'}[h_2]\\
\end{eqnarray*}
Here the first step follows from~\cite{Liu2018:inductive}, the second step from
  restricting $P'$ to be a product distribution, third from the multiplicative 
  property of $\alpha$-divergence for product distributions, fourth from the fact that
  we can split the sup of the product distributions into two parts. To simplify
  further, we use ~\cite{Liu2018:inductive} again, this time on the assumptions:
\begin{eqnarray*}
  \dist_\alpha^{\calH_1}(P_1, Q_1)& = &\inf_{P_1'}\dist_\alpha(P_1', Q_1) + 
  \sup_{h \in \calH_1} \bbE_{P_1}[h] - \bbE_{P_1'}[h] \leq \eta_1 \\
  \dist_\alpha^{\calH_2}(P_2, Q_2)& = &\inf_{P_2'}\dist_\alpha(P_2', Q_2) + 
  \sup_{h \in \calH_2} \bbE_{P_2}[h] - \bbE_{P_2'}[h] \leq \eta_2
\end{eqnarray*}
  Because $h$ contains constant functions, we know that $\sup_{h \in
  \calH}\bbE_{P_i}[h] - \bbE_{P_i'}[h] \geq 0$, and thus
  \[
    \dist_\alpha^{\calH_i}(P_i, Q_i) \leq \eta_i
  \]
Continuing, 
\begin{eqnarray*}
  \dist_{\alpha}^\calH(P, Q)
  &\leq& \inf_{P'=P_1' \otimes P_2'} a\eta_1\eta_2 +
  \dist_\alpha(P_1', Q_1) + \dist_\alpha(P_2', Q_2) \\
  &   & + \sup_{h_1 \in \calH_1}\bbE_{P_1}[h_1] - \bbE_{P_1'}[h_1]
      + \sup_{h_2 \in \calH_2}\bbE_{P_2}[h_2] - \bbE_{P_2'}[h_2]\\
  &\leq& a\eta_1\eta_2 + \inf_{P'=P_1' \otimes P_2'} 
  \dist_\alpha(P_1', Q_1) + \sup_{h_1 \in \calH_1}\bbE_{P_1}[h_1] - \bbE_{P_1'}[h_1]
  \\
  &   & + \dist_\alpha(P_2', Q_2)
        + \sup_{h_2 \in \calH_2}\bbE_{P_2}[h_2] - \bbE_{P_2'}[h_2]\\
  &\leq& a\eta_1\eta_2 + \eta_1 + \eta_2
\end{eqnarray*}
  This means $\renyidiv^\calH(P,Q) \leq
  \frac{1}{\alpha-1}\log(a(a\eta_1\eta_2 + \eta_1 + \eta_2) + 1)$. We can
  simplify this:
\begin{align*}
  \frac{1}{\alpha-1}\log(a(a\eta_1\eta_2 + \eta_1 + \eta_2) + 1) &=
  \frac{1}{\alpha-1}(\log(a\eta_1 + 1) + \log(a\eta_2+1)) \\
  &= \epsilon_1 + \epsilon_2
\end{align*}
\end{proof}

\begin{proof}
  (Of Theorem~\ref{thm:parallelcomposition}).
  Let $D$ and $D'$ be two datasets which differ in the value of a single row.
  Then, $D=(D_1,D_2)$, and we have two cases for $D'$: $D' = (D_1, D_2')$ or
  $(D_1', D_2)$ where the pairs $D_1,D_1'$ and $D_2,D_2'$ differ in one row.
  Suppose the first case is true. Then,
  $(P_1,P_2) = (A(D_1), B(D_2))$ and $(Q_1,Q_2) = (A(D_1), B(D_2'))$.
  Importantly, we have $P_1 = Q_1$. Then, letting $P = P_1 \otimes P_2$, $Q =
  Q_1 \otimes Q_2$, and $a = \alpha^2-\alpha$,
\begin{eqnarray*}
  \dist_{\alpha}^{\calH}(P, Q) & = & \inf_{P'} \dist_\alpha(P', Q) + \sup_{h \in
  \calH}\bbE_{P}[h] - \bbE_{P'}[h]\\
  & \leq & \inf_{P'=P_1 \otimes P_2'} \dist_\alpha(P', Q) + \sup_{h \in
  \calH}\bbE_{P}[h] - \bbE_{P'}[h]\\
  & = & \inf_{P'=P_1 \otimes P_2'} a\dist_\alpha(P_1, Q_1)\dist_\alpha(P_2', Q_2) +
  \dist_\alpha(P_1, Q_1) + \dist_\alpha(P_2', Q_2) + \sup_{h \in
  \calH}\bbE_{P}[h] - \bbE_{P'}[h]\\
  & = & \inf_{P'=P_1 \otimes P_2'} 
  \dist_\alpha(P_2', Q_2) + \sup_{h \in \calH}\bbE_P[h] - \bbE_{P'}[h] \\
  & \leq & \inf_{P_2'} \dist_\alpha(P_2', Q_2) +
  \sup_{h_1 \in \calH_1}\bbE_{P_1}[h_1] - \bbE_{P_1}[h_1]
      + \sup_{h_2 \in \calH_2}\bbE_{P_2}[h_2] - \bbE_{P_2'}[h_2]\\
  & = & \inf_{P_2'} \dist_\alpha(P_2', Q_2) +
      \sup_{h_2 \in \calH_2}\bbE_{P_2}[h_2] - \bbE_{P_2'}[h_2]\\
  & = & \dist_\alpha^{\calH_2}(P_2,Q_2)
\end{eqnarray*}
Here the first step follows from~\cite{Liu2018:inductive}, the second step from
  restricting $P'$ to be a product distribution, third from the multiplicative 
  property of $\alpha$-divergence for product distributions, fourth from the
  fact that $D(P_1,Q_1) = 0$ when $P_1=Q_1$ for any divergence, fifth from splitting the
  sup into two parts, sixth from further cancellation.
  For the second case, where $D' = (D_1,D_2')$, we can prove 
  $\dist_\alpha^\calH(P,Q) \leq \dist_\alpha^{\calH_1}(P_1,Q_1)$ via a similar
  argument. With a simple transformation from $\alpha$ to
  $\alpha$-Renyi divergence, we obtain our result.
\end{proof}

\section{Mechanisms}\label{sec:mechanisms-app}
\subsection{KL, Unbounded}

\begin{theorem}[Laplace Mechanism under KL]\label{thm:lap-kl}
Let $P$ and $Q$ be Laplace distributions with mean $0$ and $1$ and parameter $1/\epsilon$. Then,
\[ \KL(P, Q) = \epsilon - 1 + e^{-\epsilon} \]
\end{theorem}

\begin{proof}

We divide the real line into three intervals: $I_1 = [-\infty, 0]$, $I_2 = [0, 1]$, $I_3 = [1, \infty]$.

For any $x \in I_1$, $P(x)/Q(x) = e^{\epsilon}$, and $\Pr(I_1) = 1/2$ (under $P$). Therefore,
\[ \bbE_P[ \log (P/Q), x \in I_1 ] = \frac{1}{2} \epsilon \]

Similarly for any $x \in I_3$, $P(x)/Q(x) = e^{-\epsilon}$, and $\Pr(I_3)$ under $P$ is calculated as follows:
\[ \Pr(I_3) = \int_{1}^{\infty} \frac{1}{2} \epsilon e^{-\epsilon x} dx = \frac{1}{2} e^{-\epsilon} \]
Therefore,
\[ \bbE_P[ \log (P/Q), x \in I_3 ] = -\frac{1}{2} \epsilon e^{-\epsilon} \]

We are now left with $I_2$. For any $x \in I_2$, we have $P(x)/Q(x) = e^{-\epsilon x} / e^{-\epsilon(1 - x)} = e^{\epsilon(1 - 2x)}$. Therefore,
\begin{eqnarray*}
\bbE_P[ \log (P/Q), x \in I_2] & = & \int_{0}^{1} \frac{1}{2} \epsilon^2 (1 - 2 x) e^{- \epsilon x} dx \\
& = &  \frac{1}{2} \epsilon^2 \left( \int_{0}^{1} e^{-\epsilon x} dx - \int_{0}^{1} 2x e^{-\epsilon x} dx \right) \\
& = &  \frac{1}{2} \epsilon^2 \left( \frac{e^{-\epsilon x}}{-\epsilon} \Bigg{|}_0^1 - \frac{2x e^{-\epsilon x}}{-\epsilon} \Bigg{|}_0^1 + \int_{0}^{1} \frac{2 e^{-\epsilon x}}{-\epsilon} \Bigg{|}_0^1 dx \right) \\
& = & \frac{1}{2} \epsilon^2 \left( \frac{1 - e^{-\epsilon}}{\epsilon} + \frac{2 e^{-\epsilon}}{\epsilon} - \int_{0}^{1} 2 e^{-\epsilon x}{\epsilon} dx \right) \\
& = & \frac{1}{2} \epsilon^2 \left(  \frac{1 + e^{-\epsilon}}{\epsilon} - \frac{2 e^{-\epsilon x}}{-\epsilon^2} \Bigg{|}_0^1 \right) \\
& = & \frac{1}{2} \epsilon^2 \left(  \frac{1 + e^{-\epsilon}}{\epsilon} + \frac{2 e^{-\epsilon} - 2}{\epsilon^2} \right)
\end{eqnarray*}

Summing up the three terms, we get:
\[ \bbE_P[\log (P/Q)] = \frac{1}{2} \epsilon -\frac{1}{2} \epsilon e^{-\epsilon} + \frac{1}{2} \epsilon (1 + e^{-\epsilon}) + e^{-\epsilon} - 1 = \epsilon - 1 + e^{-\epsilon} \]

\end{proof}
The proof for the Gaussian Mechanism appears in Theorem~\ref{thm:gaussian-kl}.

\subsection{KL, Linear-Bounded}

\begin{lemma}\label{lem:kldivdual}
	Let $\calX$ be an instance space and $\phi$ be a vector of feature functions on
  $\calX$ of length $M$. 
  Let $\lin$ be the class of functions:
	\[ \lin = \{ a \phi(x) + b | a \in \bbR^M, b \in \bbR \} \]
	Then, for any two distributions $P$ and $Q$ on $\calX$, we have:
  \begin{align*}
    \KL^{\calH}(P, Q) &= \sup_{a \in \bbR^M} a^{\top} \bbE_{x \sim P}[\phi(x)] - 
    \bbE_{x \sim Q} [e^{a^{\top-1}\phi(x) }] \\
    \KL^{\lin}(P, Q) &= \sup_{a \in \bbR^M} a^{\top} \bbE_{x \sim P}[\phi(x)] - 
    \log \bbE_{x \sim Q} [e^{a^{\top}\phi(x) }] 
  \end{align*}
\end{lemma}

\begin{proof}
  For KL-divergence, we have $f(x) = x\log(x)$. This means
  \begin{align*}
    f^*(s) = \sup_{x \in \bbR} xs - x\log x
  \end{align*}
  The argument is maximized when $x = e^{s-1} $, so $f^*(s) = e^{s-1}$, and we
  obtain
  \[
  \KL^\calH(P, Q) = \sup_{h \in \calH} \bbE_{x \sim P} [h(x)] - \bbE_{x \sim
    Q}[e^{h(x)-1}]
  \]

  Now, we plug $\lin$ into $\calH$:
	\[ \KL^{\lin}(P, Q) = \sup_{a \in \bbR^M, b \in \bbE} a^{\top} \bbE_{x \sim P}[\phi(x)] + b - \bbE_{x \sim Q} [e^{a^{\top} \phi(x) + b - 1}] \]
	Differentiating the objective with respect to $b$, we have that at the optimum: 
	\[ 1 - e^{b - 1} \bbE_{x \sim Q} [e^{a^{\top} \phi(x)}] = 0, \]
	which means that the optimum $b$ is equal to:
	\[ b = 1 - \log \bbE_{x \sim Q} [e^{a^{\top} \phi(x)}] \]
	Plugging this in the objective, we get that:
	\begin{eqnarray*}
		\KL^{\calH}(P, Q) & = & \sup_{a \in \bbR^M} a^{\top} \bbE_{x \sim P} [\phi(x)] + 1 - \log \bbE_{x \sim Q} [e^{a^{\top} \phi(x)}] + (\bbE_{x \sim Q} [e^{a^{\top} \phi(x)}])^{-1} \bbE_{x \sim Q} [e^{a^{\top} \phi(x)}] \\
		&  = &  \sup_{a \in \bbR^M} a^{\top} \bbE_{x \sim P} [\phi(x)] - \log \bbE_{x \sim Q} [e^{a^{\top} \phi(x)}] 
	\end{eqnarray*}
	The lemma follows.
\end{proof}

\begin{theorem}[Laplace mechanism under $(lin, \KL)$]\label{thm:lap-linkl}
Let $P = Lap(0, 1/\epsilon)$ and $Q = Lap(1, 1/\epsilon)$. Then, 
\[ \KL^\lin(P, Q) = \log\left( 1 - \left( \frac{1 - \sqrt{1 + \epsilon^2}}{\epsilon} \right)^2 \right) + \sqrt{1 + \epsilon^2} - 1 \]
\end{theorem}

\begin{proof}
Recall that the density function of $P$ (resp. $Q$) is
  $\frac{\epsilon}{2}e^{-\epsilon |x|}$ (resp. $\frac{\epsilon}{2} e^{-\epsilon
  | x - 1|}$). By Lemma~\ref{lem:kldivdual}, computing the linear KL is equivalent to solving the following problem:
\[ \max_a a \bbE_{x \sim P}[x] - \log \bbE_{x \sim Q}[e^{ax}] = \max_a - \log \left( \frac{e^{a}}{1 - a^2/\epsilon^2} \right) , a \in [ - \epsilon, \epsilon ] = \max_a \log (1 - a^2/\epsilon^2) - a, a \in [-\epsilon, \epsilon] \]
Note that the expression $\bbE_{x \sim Q}[e^{ax}]$ blows up to $\infty$ for $a \notin [-\epsilon, \epsilon]$, and hence the maximizer $a$ has to lie in $[-\epsilon, \epsilon]$. Taking the derivative and setting it to $0$, we get:
\[ \frac{-2a/\epsilon^2}{1 - a^2/\epsilon^2} - 1 = 0, \]
which, after some algebra, becomes the quadratic equation:
\[ a^2 - 2a - \epsilon^2 = 0 \]
The roots of this equation are: $a = 1 \pm \sqrt{1 + \epsilon^2}$. The first
  root is more than $\epsilon$, and hence we choose $a = 1 - \sqrt{1 +
  \epsilon^2}$ as the solution. Plugging this solution into the expression for
  $\KL^\lin$, we get:
\[ \KL^\lin(P, Q) = \log\left( 1 - \left( \frac{1 - \sqrt{1 + \epsilon^2}}{\epsilon} \right)^2 \right) + \sqrt{1 + \epsilon^2} - 1 \]
\end{proof}

\begin{theorem}[Gaussian mechanism under (lin)-KL]\label{thm:gaussian-kl}
	Let $P = \calN(\mu_1, \sigma_1^2)$ and $Q = \calN(\mu_2, \sigma_2^2)$. Then,
	\[ \KL^\lin(P, Q) = \frac{(\mu_1 - \mu_2)^2}{2\sigma_2^2} \]
\end{theorem}

\begin{proof}
	By definition,
	\[ \KL^\lin(P, Q) = \sup_{a} a \mu_1 - \log e^{a \mu_2 + a^2 \sigma_2^2/2} = \sup_a a (\mu_1 - \mu_2) - \frac{1}{2} a^2 \sigma_2^2 \]
	Differentiating wrt $a$ and setting the derivative to $0$, the optimum is achieved at $a = (\mu_1 - \mu_2)/\sigma_2^2$, at which the optimal value is $(\mu_1 - \mu_2)^2/2 \sigma_2^2$.
\end{proof}

\subsection{Renyi, Unbounded}

\begin{theorem}[Laplace Mechanism under $\alpha$-Renyi]\label{thm:lap-renyi}
  Let $P$ and $Q$ be i.i.d. Laplace distributions in $d$ dimensions 
  with mean 0 (resp. $v=(v_1,v_2,\ldots, v_d)$) and parameter
  $\frac{1}{\epsilon}$. Then,
  \[
    \renyidiv(P,Q) = \frac{1}{\alpha-1}\sum_{i=1}^d \log\left(
      \left(\frac{1}{2} + \frac{1}{4\alpha-2}\right)e^{v_i(\alpha-1)\epsilon}
      + \left(\frac{1}{2} - \frac{1}{4\alpha-2} \right)e^{-v_i\alpha\epsilon}
    \right)
  \]
\end{theorem}

\begin{proof}
  We first compute
  \[
    \dist_\alpha(P,Q) = \frac{1}{\alpha^2-\alpha} 
      \left(\int_{\bbR^d} \left(\frac{dP}{dQ}\right)^{\alpha}dQ - 1 \right)
  \]
We write the integral as a product. Let $p_i$ be the p.d.f. for the $Lap(i,
  \frac{1}{\epsilon})$ distribution:
  \begin{align*}
    \int_{\bbR^d}\left(\frac{dP}{dQ}\right)^\alpha dQ &= 
    \int_{\bbR^d}P(x)^\alpha Q(x)^{1-\alpha} dx \\ 
    &=
    \int_{\bbR^d}\left(\prod_{i=1}^dp_0(x_i)^\alpha p_{v_i}(x_i)^{1-\alpha}\right)
    dx \\
    &= \prod_{i=1}^d\int_{\bbR} p_0(x_i)^\alpha p_{v_i}(x_i)^{1-\alpha} dx_i
  \end{align*}
  We will compute each integral in the product individually. For the first case,
  suppose $v_i > 0$. We now split 
  the real line into three regions: $I_1=[\infty,0]$, $I_2
  = [0,v_i]$, and $I_3 = [v_i,\infty]$. Recall that $p_i(x) =
  \frac{\epsilon}{2}e^{-|x-i|\epsilon}$.
   \begin{align*}
     \int_{-\infty}^0p_0(x)^\alpha p_{v_i}(x)^{1-\alpha}dx 
     &= \frac{\epsilon}{2}\int_{-\infty}^0 \left(\frac{
       e^{x\epsilon} }{ e^{(x-v_i)\epsilon} } \right)^\alpha
           e^{(x-v_i)\epsilon}dx \\
        &= \frac{\epsilon}{2}\int_{-\infty}^0 e^{v_i\alpha\epsilon-v_i\epsilon+x\epsilon} dx =
        \frac{1}{2}e^{v_i(\alpha-1)\epsilon} \\
     \int_0^{v_i}p_0(x)^\alpha p_{v_i}(x)^{1-\alpha}dx
     &= \frac{\epsilon}{2}\int_0^{v_i} \left( \frac{
       e^{-x\epsilon} }{ e^{(x-v_i)\epsilon} } \right)^\alpha
           e^{(x-v_i)\epsilon}dx \\
        &= \frac{\epsilon}{2}\int_{0}^{v_i} e^{(1-2\alpha)x\epsilon+v_i(\alpha-1)\epsilon} dx = 
        \frac{1}{2-4\alpha}e^{v_i(\alpha-1)\epsilon}( e^{v_i(1-2\alpha)\epsilon}-1 )\\
     \int_{v_i}^\infty p_0(x)^\alpha p_{v_i}(x)^{1-\alpha} &=
     \frac{\epsilon}{2}\int_{v_i}^\infty \left(\frac{
       e^{-x\epsilon} }{ e^{(v_i-x)\epsilon} } \right)^\alpha
           e^{(v_i-x)\epsilon}dx \\
        &= \frac{\epsilon}{2}\int_{v_i}^\infty e^{-v_i\alpha\epsilon+v_i\epsilon-x\epsilon} dx =
        \frac{1}{2}e^{-v_i\alpha\epsilon} \\
     \int_\bbR p_0(x)^\alpha p_{v_i}(x)^{1-\alpha}dx &= \frac{1}{2}e^{v_i(\alpha-1)\epsilon} +
     \frac{1}{2}e^{-v_i\alpha\epsilon} + 
        \frac{1}{2-4\alpha}(e^{-v_i\alpha\epsilon}-e^{v_i(\alpha-1)\epsilon}) \\
        &= \left(\frac{1}{2} + \frac{1}{4\alpha-2}\right)e^{v_i(\alpha-1)\epsilon}
           + \left(\frac{1}{2} - \frac{1}{4\alpha-2}\right)e^{-v_i\alpha\epsilon}
  \end{align*} 
  When $v_i < 0$, we get the same answer, with $v_i$ replaced by $-v_i$.
  Let $F(x) = \left(\frac{1}{2} + \frac{1}{4\alpha-2}\right)e^{(\alpha-1)|x|}
           + \left(\frac{1}{2} - \frac{1}{4\alpha-2}\right)e^{-\alpha |x|}$.
  We can write
  \begin{equation}\label{eq:lapalphadiv}
    \dist_{\alpha}(P,Q) = \frac{1}{\alpha^2-\alpha}\left(\prod_{i=1}^dF(v_i\epsilon)
    - 1 \right)
  \end{equation}
  The expression for $\renyidiv(P,Q)$ follows easily.
\end{proof}

\begin{corollary}
  If $\alpha \geq 1$, then $\renyidiv(P,Q) \geq \epsilon\|v\|_1$, where $v =
  (v_1,\ldots, v_d)$.
\end{corollary}
\begin{proof}
  When $\alpha > 1$, then $e^{(\alpha-1)|x|} > e^{-\alpha |x|}$.
  Thus, $F(x)$, defined above, is lower bounded by
  $\frac{1}{2}e^{(\alpha-1)|x|}$. Plugging into Equation~(\ref{eq:lapalphadiv}),
  we get $\dist_\alpha(P,Q) \geq
  \frac{1}{\alpha^2-\alpha}\left(e^{\epsilon(\alpha-1)\|v\|_1}-1\right)$. The result for
  $\renyidiv(P,Q)$ follows easily.
\end{proof}

\begin{theorem}[Gaussian mechanism under $\alpha$-Renyi]
  \label{thm:gaussian-renyi}
  Let $P$ and $Q$ be i.i.d. Normal distributions in $d$ dimensions with mean $0$
  (resp. $v=(v_1,v_2,\ldots, v_d)$) and variance $\sigma^2$. Then,
  $\renyidiv(P,Q) = \frac{\alpha\|v\|_2^2}{2\sigma^2}$.
\end{theorem}

\begin{proof}
  We will compute
  \[
    \dist_\alpha(P,Q) = \frac{1}{\alpha^2-\alpha} 
      \left(\int_{\bbR^d} \left(\frac{dP}{dQ}\right)^{\alpha}dQ - 1 \right)
  \]
  Just like Theorem~\ref{thm:lap-renyi}, we can write
  \[
    \int_{\bbR^d}\left(\frac{dP}{dQ}\right)^\alpha dQ = 
    \prod_{i=1}^d\int_{\bbR} p_0(x_i)^\alpha p_{v_i}(x_i)^{1-\alpha} dx_i
  \]
  where $p_i(x)$ is the p.d.f. of $\calN(i, \sigma^2)$. Therefore,
  \begin{align*}
    \int_{\bbR} p_0(x)^\alpha p_{v_i}(x)^{1-\alpha}dx &=
    \frac{1}{\sqrt{2\pi\sigma^2}}\int_{-\infty}^\infty \left( \frac{
      e^{-x^2/(2\sigma^2)} }
    { e^{-(x-v_i)^2/(2\sigma^2)} }\right)^\alpha e^{-(x-v_i)^2/(2\sigma^2)} dx \\
    &= \frac{1}{\sqrt{2\pi\sigma^2}}\int_{-\infty}^\infty 
        e^{( -(x+v_i(\alpha-1))^2 + v_i^2(\alpha-1)^2+v_i^2(\alpha-1))/(2\sigma^2)}dx
    \\
    &= e^{v_i^2((\alpha-1)^2 + (\alpha-1))/(2\sigma^2)} =
    e^{v_i^2(\alpha^2-\alpha)/(2\sigma^2)}
  \end{align*}
  Therefore,
  \[
    \dist_\alpha(P,Q) = \frac{1}{\alpha^2-\alpha}\left(\prod_{i=1}^d
    e^{v_i^2(\alpha^2-\alpha)/(2\sigma^2)} -1 \right) = \frac{1}{\alpha^2-\alpha}
    (e^{\|v\|_2^2(\alpha^2-\alpha)/(2\sigma^2)}-1)
  \]
  The result for $\renyidiv(P,Q)$ follows immediately.
\end{proof}

\subsection{Renyi, Linear-Bounded}

\begin{lemma}[Non-private Release]\label{lemma:non-private}
  Let $A$ be a mechanism such that there exist two datasets $D$ and $D'$ for
  which $A(D)$ and $A(D')$ are different point masses. Let $\calH$ be a function
  class that contains linear functions. Then, $\renyidiv(A(D), A(D')) = \infty$.
\end{lemma}

\begin{proof}
	Let $P$ denote the distribution of $A(D)$ and $Q$ denote the distribution of
  $A(D')$.
  If we show that $\dist^\calH_\alpha(P,Q) = \infty$, then certainly
  $\renyidiv^\calH(P,Q) = \infty$.
  Observe that $\dist^\calH_{\alpha}(P,Q) \geq \dist^\lin_\alpha(P,Q)$ by assumption.
  Hence, we are done if we show that $\dist^\lin_\alpha(P,Q) = \infty$.
  \[ \dist^\lin_{\alpha}(P, Q) \geq \sup_{a, b \in \bbR} \bbE_{x \sim P} [ ax + b] - 
  C_\alpha \bbE_{x \sim Q} [|ax + b|^{\alpha/(\alpha-1)}] \]
	Suppose $P$ and $Q$ are point masses at $x = p$ and $x = q$ respectively.
  Then, there exists an $a$ and a $b$ such that $a p + b > 0$ and $a q + b = 0$.
  Pick any $\lambda > 0$. Plugging into the above,
	\[ \dist^\lin_\alpha(P, Q) \geq \lambda (a p + b) - \lambda (a q + b) 
  = \lambda (a p + b) \]
	Observe that as $ap + b > 0$ and is fixed with $\lambda$, $\lambda (a p + b) \rightarrow \infty$ as $\lambda \rightarrow \infty$. The lemma follows. 
\end{proof}

\begin{lemma}\label{lem:alphadivdual}
  Let $\calX$ be an instance space and $\phi$ be a vector of feature functions on
  $\calX$ of length $M$. Then, for any two distributions $P$ and $Q$ on $\calX$,
  we have:
  \begin{align*}
    \dist_{\alpha}^\calH(P,Q) &= \sup_{h \in \calH} \bbE_{x\sim P}[h(x)]
    - C_\alpha \bbE_{x\sim Q}[|h(x)|^{\alpha/(\alpha-1)}] -
    \frac{1}{\alpha^2-\alpha} \\
    \dist_{\alpha}^\lin(P,Q) &= \sup_{a \in \bbR^n, b \in \bbR} \bbE_{x\sim P}[a^Tx+b]
    - C_\alpha \bbE_{x\sim Q}[|a^Tx+b|^{\alpha/(\alpha-1)}] -
    \frac{1}{\alpha^2-\alpha} \\
  \end{align*}

  where $C_\alpha = \frac{(\alpha-1)^{\alpha/(\alpha-1)}}{\alpha}$.
\end{lemma}
\begin{proof}
  We need to compute $f^*(s) = \sup_{x \in \bbR} xs - f(x)$ 
  for $f = \frac{|x|^\alpha-1}{\alpha^2-\alpha}$ 
  Setting the derivative of $f$ to zero, we get:
	\[
   	s - \sgn(x)\frac{|x|^{\alpha-1}}{\alpha-1} = 0 \implies x =
	  (|s|(\alpha-1))^{1/(\alpha-1)}
	\]
  Therefore,
	\[
		f^*(s) = \frac{1}{\alpha}(|s|(\alpha-1))^{\alpha/(\alpha-1)} -
    \frac{1}{\alpha^2-\alpha} = C_\alpha|s| - \frac{1}{\alpha^2-\alpha}
	\]
  This allows us to derive the expressions for $\dist_\alpha^\lin$ and
  $\dist_\alpha^\calH$ by plugging into
  \[
    \dist_\alpha^\calH(P,Q) = \sup_{h \in \calH}\bbE_{x \sim P}[h(x)] - \bbE_{x \sim
    Q}[f^*(h(x))]
  \]
\end{proof}

\begin{lemma}\label{lem:divdual-sym}
  Suppose $P$ is a $d$-dimensional r.v. such that $\sgn(\bbE_{x \sim P}[x]) =
  (e_1, e_2,\ldots e_d)$ and
  $Q$ is a $d$-dimemsional r.v. with independent coordinates and marginals 
  symmetric about 0. Then, the variational form of $\dist_{\alpha}^{\lin}(P,Q)$
  can be written as
  \[
    \sup_{a \in \bbR^d, \sgn(a_i)=e_i, b \in \bbR, b \geq 0}
    \bbE_{x \sim P}[ax] + b - C_\alpha\bbE_{x \sim
    Q}[|ax+b|^{\alpha/(\alpha-1)}] - \frac{1}{\alpha^2-\alpha}
  \]
  where $C_\alpha = \frac{(\alpha-1)^{\alpha/(\alpha-1)}}{\alpha}$.
\end{lemma}
\begin{proof}
  By Lemma~\ref{lem:alphadivdual},
  \begin{equation}\label{eq:alphadivdual2}
    \dist^{\lin}_{\alpha}(P , Q) 
    = \sup_{a \in \bbR^d,b \in \bbR} \bbE_{x \sim P}[a^Tx] + b
      - C_\alpha\bbE_{x \sim Q}[|a^Tx+b|^{\alpha/(\alpha-1)}]
      - \frac{1}{\alpha^2-\alpha}
  \end{equation}
  Let $a = (a_1,a_2,\ldots, a_d)$.
  The distribution of $|a^TQ+b|$ is determined only by the magnitude of each
  $a_i$, not its sign, because of the symmetry of $Q$. The sign of $b$ does not
  matter, either, as $|a^TQ-b| = |(-a)^TQ+b| = |a^TQ+b|$.
  Therefore,~(\ref{eq:alphadivdual2}) achieves its maximum value when
  $\sgn(a_i) = e_i$ and $b > 0$, and we are done.
\end{proof}
\begin{lemma}\label{lem:dual-max}
  The function $f(a) = a-Xa^{\alpha/(\alpha-1)}$
  has a global maximum of
  \[
    \frac{(\alpha-1)^{\alpha-1}}{\alpha^{\alpha}X^{\alpha-1}}
  \]
\end{lemma}
\begin{proof}
  We observe that $f(a)$ is concave down over all real numbers. 
  Its derivative vanishes when $a =
  \frac{(\alpha-1)^{\alpha-1}}{\alpha^{\alpha-1}X^{\alpha-1}}$,
  and because of the first observation,
  this is the global maximum. This is equal to
  \[
    \frac{(\alpha-1)^{\alpha-1} }{\alpha^{\alpha-1} X^{\alpha-1}} -
    \frac{(\alpha-1)^\alpha}{\alpha^\alpha X^{\alpha-1}}
    =
    \frac{(\alpha-1)^{\alpha-1}}{\alpha^{\alpha}X^{\alpha-1}}
  \]
\end{proof}

\begin{theorem}\label{thm:linearbound}
  (Symmetric Distributions under lin $\alpha$-Renyi): Suppose $P = X$ and $Q=X+v$
  where $X$ is a $d$-dimensional r.v. consisting of $d$ i.i.d samples from an
  underlying distribution $Y$, $v \in \bbR^d$, and $Y$ is symmetric around 0 such that
  $\bbE[|Y|^{\alpha/(\alpha-1)}] = K$. Then,
  \[
     \renyidiv^{\lin} (P, Q) \leq 
     \frac{1}{\alpha-1}\log\left(1+\frac{\|v\|_\alpha^\alpha}{(0.5^dK)^{\alpha-1}}\right)
  \]
\end{theorem}
\begin{proof}
  We apply Lemma~\ref{lem:divdual-sym}. For simplicity, let $A = \frac{\alpha}{\alpha-1}$.
  Then,
  \begin{equation}\label{eq:opti}
    \dist_{\alpha}^{\lin}(P , Q) =
    \sup_{a \in \bbR^d, \sgn(a_i) = \sgn(v_i), b \in \bbR,b\geq 0} a^Tv+b - C_{\alpha}
    \bbE_{x \sim Q}[|a^Tx+b|^{A}] - \frac{1}{\alpha^2-\alpha}
  \end{equation}

  Because $Q$ is symmetric, we can write
  \begin{align*}
    \bbE_{x \sim Q}[|a^Tx+b|^{A}] &\geq \frac{1}{2^d}\bbE_{x \sim Q\vert
    \sgn(x) = \sgn(a_i)}[|a^Tx+b|^{A}] \\
    &\geq \frac{1}{2^d}\bbE_{x \sim Q \vert
    \sgn(x_i) = \sgn(a_i)}\left[\sum_{i=1}^d(a_ix_i)^{A}\right]
    + b^{A} \\
    &= \frac{1}{2^d}\sum_{i=1}^d
    \bbE_{x_i \sim Y \vert \sgn(x_i) = \sgn(a_i)}
    [(a_ix_i)^{A}] + b^{A} \\
    &= \frac{1}{2^d} \sum_{i=1}^d |a_i|^A
    \bbE_{x_i \sim Y \vert \sgn(x_i) = \sgn(a_i)}
    [|x_i|^{A}] + b^{A} \\
    &= \frac{1}{2^d} \sum_{i=1}^d |a_i|^A K + b^A
  \end{align*}
  Here, the first step comes from discarding the parts of the expectation where
  $\sgn(a_i) \neq \sgn(x)$ which is possible because the argument of the
  expectaion is always positive. Finally, the set that remains has measure
  measure $\frac{1}{2^d}$, so we normalize accordingly.
  The second step comes from the fact that $|a^Tx+b|^A > \sum_{i=1}^d(a_ix_i)^A +
  b^A$ when $b, a_ix_i > 0$; the third comes from linearity of
  expecation; the fourth from the fact that $a_i$ and $x_i$ have the same sign;
  and the fifth from the fact that $Y$ is symmetric.
  We can now plug into~(\ref{eq:opti}) and simplify.
  \begin{align*}
    \dist_\alpha^{\lin}(P,Q) &\leq \sup_{a\in \bbR^d,\sgn(a_i)=\sgn(v_i),b \in \bbR,
    b \geq 0} a^Tv+b
    - C_\alpha \frac{1}{2^d} \sum_{i=1}^d |a_i|^A K - C_\alpha b^A 
    - \frac{1}{\alpha^2-\alpha}\\
    &\leq \sum_{i=1}^d \sup_{a_i \in \bbR,\sgn(a_i)=\sgn(v_i) } a_iv_i -
    \frac{C_\alpha K}{2^d} |a_i|^A + \sup_{b \in \bbR, b > 0} b- C_\alpha b^A
    - \frac{1}{\alpha^2-\alpha}
    \\
    &\leq \sum_{i=1}^d |v_i| \sup_{a_i \in \bbR, a_i > 0} a_i -
    \frac{C_\alpha K}{2^d|v_i|} a_i^A + \sup_{b \in \bbR, b > 0} b- C_\alpha b^A
    - \frac{1}{\alpha^2-\alpha} \\
    &= \sum_{i=1}^d |v_i|
    \frac{(\alpha-1)^{\alpha-1}}{\alpha^{\alpha}}
    \frac{|v_i|^{\alpha-1}(2^d)^{\alpha-1}}{C_\alpha^{\alpha-1}K^{\alpha-1}}
    + \frac{(\alpha-1)^{\alpha-1}}{\alpha^\alpha C_\alpha^{\alpha-1} }
    - \frac{1}{\alpha^2-\alpha}\\
    &= \sum_{i=1}^d |v_i|
    \frac{1}{\alpha^2-\alpha}
    \frac{|v_i|^{\alpha-1}(2^d)^{\alpha-1}}{K^{\alpha-1}}
    + \frac{1}{\alpha^2-\alpha} - \frac{1}{\alpha^2-\alpha} \\
    &= \frac{2^{d(\alpha-1)}\|v\|_\alpha^\alpha}{(\alpha^2-\alpha)K^{\alpha-1}}
  \end{align*}
  Here, the second line comes from the fact that a sup of a sum is at most the
  sum of sups, the third is from pulling out $v_i$ from the sup and being
  careful about signs, the fourth from applying Lemma~\ref{lem:dual-max}, the fifth
  from plugging in for the $C_\alpha$ term and simplification,
  and the sixth from writing the answer in terms of norms.
  \end{proof}
  \begin{corollary}\label{thm:linearbound-coro}
    For all $\alpha > 2$, the following are true:
    \begin{align*}
      \renyidiv^\lin(\calN(0, \sigma^2I_d), \calN(v, \sigma^2I_d))
        &\leq \frac{1}{\alpha-1} \log\left( 1 +
      \frac{\|v\|_\alpha^\alpha}{(0.5^d \times \sqrt{2/\pi}
      )^{\alpha-1}\sigma^\alpha}\right) \\
      \renyidiv^\lin(Lap_d(0, 1/\epsilon), Lap_d(v,
      1/\epsilon)) &\leq \frac{1}{\alpha-1} \log\left( 1 +
      \frac{\|v\|_\alpha^\alpha \epsilon^\alpha}{(0.5^d
      )^{\alpha-1}}\right)
    \end{align*}
  \end{corollary}
  \begin{proof}
    When $\alpha \geq 2$, $\frac{\alpha}{\alpha-1}$ is between 1 and
    2, a rather small range. If $Y = \calN(0, \sigma^2)$, then 
    \[
      K =
      \sigma^{\alpha/(\alpha-1)}\bbE[|\tilde{Y}|^{\alpha/(\alpha-1)}] \geq
      \sigma^{\alpha/(\alpha-1)}\inf_{1 \leq \gamma \leq 2}
      \bbE[|\tilde{Y}|^{\gamma}]
    \]
    where $\tilde{Y} = \calN(0, 1)$.
    $\bbE[|\tilde{Y}|^\gamma]$ is minimized when $\gamma=1$, and we get $K \geq
    \sigma^{\alpha/(\alpha-1)} \sqrt{\frac{2}{\pi}} \approx
    0.79 \sigma^{\alpha/(\alpha-1)}$. We then apply Theorem~\ref{thm:linearbound}.

    If $Y = Lap(0, \frac{1}{\epsilon})$, then
    \[
      K =
      \epsilon^{\alpha/(\alpha-1)}\bbE[|\tilde{Y}|^{\alpha/(\alpha-1)}] \geq
      \frac{1}{\epsilon^{\alpha/(\alpha-1)}}\inf_{1 \leq \gamma \leq 2}
      \bbE[|\tilde{Y}|^{\gamma}]
    \]
    where $\tilde{Y} = Lap(0, 1)$.
    $\bbE[|\tilde{Y}|^\gamma]$ is minimized when $\gamma=1$, and we get $K \geq
    \frac{1}{\epsilon^{\alpha/(\alpha-1)}} $. We then apply
    Theorem~\ref{thm:linearbound}.
  \end{proof}
  
\begin{proof} (Of Theorem~\ref{thm:matrix-mechanism})
  Let $x=(x_1,x_2,\ldots,x_n)$, and the columns of $A$ be
  $(a_1,a_2,\ldots,a_n)$. Changing $D$ in one place results in a change by $1$
  in at most one $x_i$.
  Thus, $Ax = \sum_{i=1}^n x_ia_i$ has $L_1$ sensitivity $\|A\|_1$. We can use
  the multidimensional Laplace mechanism (Corollary~\ref{thm:linearbound-coro})
  which allows us to release
  $\tilde{a} = Ax + \|A\|_1Lap_s(0, 1/\epsilon)$ offering $\calH$-bounded privacy with
  parameter 
  \[
    \frac{1}{\alpha-1} \log\left(1 + \|v\|_\alpha^\alpha 
    2^{s(\alpha-1)}\frac{\epsilon^\alpha}{\|v\|_1^\alpha} \right)
  \]
  where $\calH$ is the set of linear functions : $\bbR^s \rightarrow \bbR$.
  Because $\|v\|_\alpha \leq \|v\|_1$, we can simplify this to
  $\frac{1}{\alpha-1} \log (1 + 2^{(\alpha-1)e}\epsilon^\alpha )$.
  We let $\calG$ be the set of linear functions $\bbR^s \rightarrow \bbR^d$,
  $\calH$ the set of linear functions $\bbR^s \rightarrow \bbR$, and
  $\calI$ the set of linear functions $\bbR^d \rightarrow \bbR$. Notice that
  $WA^\dag \in \calG$ and $i \circ g \in \calH$ for all $i\in \calI$, $g \in
  \calG$. Relasing $M_A(W,x,\epsilon) = WA^\dag\tilde{a}$ then satisfies $\calH$-
  capacity bounded privacy by post-processing.
\end{proof}

\section{Generalization}

\begin{proof} (Of Theorem~\ref{thm:generalize})
Let $S = \{ x_1, \ldots, x_n \}$ where $x_i$ are drawn iid from an underlying data distribution $D$. Let $S_{i \rightarrow x}$ denote $S$ with its $i$-th element $x_i$ replaced by $x$. Then, we have:
\begin{eqnarray*}
\bbE_{S \sim D^n, M}\left(\frac{1}{n} \sum_{i=1}^{n} q_S(x_i) - \bbE_{x \sim D}[q_S(x)]\right) & = & \bbE_{S \sim D^n} \frac{1}{n} \sum_{i=1}^{n} \left( \bbE_M[q_S(x_i)] - \bbE_{x \sim D, M}[q_S(x)] \right) \\
& = & \bbE_{S \sim D^n} \frac{1}{n} \sum_{i=1}^{n} \left( \bbE_M[q_S(x_i)] - \bbE_{x \sim D, M}[q_{S_{i \rightarrow x}}(x_i)] \right)
\end{eqnarray*}
Here the first step follows from algebra, and the second step follows from observing that when $S \sim D^n, x \sim D$, $q_S(x)$ has the same distribution as when $S \sim D^n, x \sim D$, $q_{S_{i \rightarrow x}}(x_i)$. 
We pick any $i$. The term:
\begin{eqnarray*}
\bbE_M[q_S(x_i)] - \bbE_M[q_{S_{i \rightarrow x}}(x_i)] & = & \bbE_{q \sim M(S)}[q(x_i)] - \bbE_{q \sim M(S_{i \rightarrow x})}[q(x_i)] \\
&  = & \bbE_{q \sim M(S)}[h_{x_i}(q)] - \bbE_{q \sim M(S_{i \rightarrow x})}[h_{x_i}(q)] \\
& \leq & \sup_{h \in \calH} \bbE_{q \sim M(S)}[h(q)] - \bbE_{q \sim M(S_{i \rightarrow x})}[h(q)] \\
&  \leq & \epsilon 
\end{eqnarray*}
Here the first step follows from simplifying notation, the second from the definition of $h_{x_i}$, the third from the fact that $\calH$ includes $h_{x_i}$ and the fourth from the fact that $\ipm^{\calH}(M(S), M(S_{i \rightarrow x})) \leq \epsilon$. The theorem follows from combining this with Theorem~\ref{thm:pinsker}.
\end{proof}

\begin{proof}(Of Theorem~\ref{thm:pinsker})
If $\calH$ is translation invariant and convex, then, we can write the $\calH$-restricted KL divergence between any  two distributions $P$ and $Q$ as follows:~\cite{Liu2018:inductive, Farnia2018:gan}
\begin{equation}\label{eqn:restricteddiv}
\KL^{\calH}(P, Q) = \inf_{\tilde{P}} \KL(\tilde{P}, Q) + \sup_{h \in \calH} \bbE_{x \sim P}[h(x)] - \bbE_{x \sim \tilde{P}}[h(x)] 
\end{equation}
Let $P'$ be the $\tilde{P}$ that achieves the infimum in~\eqref{eqn:restricteddiv}. Then, from Pinsker Inequality, the left hand side of~\eqref{eqn:restricteddiv} is at least:
\[ \frac{1}{2} (TV(P', Q))^2 + \sup_{h \in \calH} \bbE_{x \sim P} [h(x)] - \bbE_{x \sim P'}[h(x)] \]
Let $\calF$ be the class of all functions with range $[-1, 1]$; by definition of the total variation distance, and because $\calH \subseteq \calF$, we have:
\[ \ipm^{\calF}(P, Q) = 2 TV(P, Q) \geq \ipm^{\calH}(P, Q) \]
Therefore $\KL^{\calH}(P, Q)$ is at least
\begin{eqnarray*}
& \geq & \frac{1}{8} (\ipm^{\calH}(P', Q))^2 + \ipm^{\calH}(P, P') \\
& \geq & \frac{1}{16} (\ipm^{\calH}(P', Q))^2 + (\ipm^{\calH}(P, P'))^2 \\
& \geq & \frac{1}{64} (\ipm^{\calH}(P', Q) + \ipm^{\calH}(P, P'))^2 \\
& \geq & \frac{1}{64} (\ipm^{\calH}(P, Q))^2
\end{eqnarray*}
Here the first step follows from Lemma~\ref{lem:positiveipm}, and the second step because as the range of any $h$ is $[-1, 1]$, $\ipm^{\calH}(P, P') \leq 2$, and hence $\ipm^{\calH}(P, P') \geq \frac{1}{2}(\ipm^{\calH}(P, P'))^2$. The third step follows because for any $a$ and $b$, $\frac{a^2}{2} + b^2 \geq \frac{1}{8}(a + b)^2$, and the final step from the triangle inequality of IPMs. The theorem thus follows.
\end{proof}

\begin{lemma}
Let $\calH$ be a function class that is closed under negation. Then, for any two distributions $P$ and $P'$,
\[ \ipm^{\calH}(P, P') = \sup_{h \in \calH} \bbE_{x \sim P} [h(x)] - \bbE_{x \sim P'}[h(x)] \]
\label{lem:positiveipm}
\end{lemma}

\begin{proof}
Observe that $\ipm^{\calH}(P, P') \geq \sup_{h \in \calH} \bbE_{x \sim P} [h(x)] - \bbE_{x \sim P'}[h(x)]$ by definition. 

Now let $h'$ be function in $\calH$ that achieves the supremum in $\ipm^{\calH}(P, P')$. If $\bbE_{x \sim P} [h'(x)] \geq \bbE_{x \sim P} [h'(x)]$, then 
\[ \sup_{h \in \calH} \bbE_{x \sim P} [h(x)] - \bbE_{x \sim P'}[h(x)] \geq \bbE_{x \sim P} [h'(x)] - \bbE_{x \sim P'}[h'(x)] = \ipm^{\calH}(P, P'), \]
If not, then, $\bbE_{x \sim P} [-h'(x)] \geq \bbE_{x \sim P} [-h'(x)]$, and
\[ \sup_{h \in \calH} \bbE_{x \sim P} [h(x)] - \bbE_{x \sim P'}[h(x)] \geq \bbE_{x \sim P} [-h'(x)] - \bbE_{x \sim P'}[-h'(x)] = \ipm^{\calH}(P, P') \]
The lemma follows.
\end{proof}

\end{document}